\documentclass[12pt, a4paper]{article}
\usepackage[utf8]{inputenc}
\usepackage[T1]{fontenc}
\usepackage{amsmath}
\usepackage{amsthm}
\usepackage{bbm} %indicatrice
\usepackage{graphicx}
\usepackage{caption}
\usepackage{fullpage}
\usepackage{lmodern}
\usepackage{fourier}
\usepackage{comment}
\usepackage{pict2e}
\usepackage{color}
\usepackage{geometry}	
\usepackage{algorithm}
\usepackage{algorithmic}
\newtheorem{theo}{Theorem}
\newtheorem{definition}{Definition}
\newtheorem{lemma}{Lemma}
\newtheorem{proposition}{Proposition}
\newtheorem{corollary}{Corollary}
\newtheorem{rmk}{Remark}
\allowdisplaybreaks

\begin{document}

\title{1-bit Matrix Completion:
PAC-Bayesian Analysis of a Variational Approximation}
\author{Vincent Cottet
 \& 
Pierre Alquier \footnote{Corresponding author, pierre.alquier@ensae.fr. This author
gratefully acknowledges financial support from the research programme
{\it New Challenges for New Data} from LCL and GENES, hosted by the
{\it Fondation du Risque}, and from Labex ECODEC (ANR - 11-LABEX-0047).
} \\
CREST, ENSAE, Université Paris Saclay}
\maketitle

\begin{abstract}
Due to challenging applications such as collaborative filtering,
the matrix completion problem has been widely studied in the past few years.
Different approaches
rely on different structure assumptions on the matrix in hand. Here, we
focus on the completion of a (possibly) low-rank matrix with binary entries,
the so-called 1-bit matrix completion problem. Our approach relies on
tools from machine learning theory: empirical risk minimization and its
convex relaxations. We propose an algorithm to compute a variational
approximation of the pseudo-posterior. Thanks to the convex relaxation, the
corresponding minimization problem is bi-convex, and thus the method
behaves well in practice. We also study the performance of this variational
approximation through PAC-Bayesian learning bounds. On the contrary to
previous works that focused on upper bounds on the estimation error of
$M$ with various matrix norms, we are able to derive from this analysis
a PAC bound on the prediction error of our algorithm.

We focus essentially on convex relaxation through the hinge loss, for which
we present the complete analysis, a complete simulation study and a test on
the MovieLens data set. However, we also discuss a variational approximation
to deal with the logistic loss.
\end{abstract}

\section{Introduction}

Motivated by modern applications like recommendation systems and collaborative
filtering, video analysis or quantum statistics, the matrix completion problem
has been widely studied in the recent years.
Recovering a matrix which is, without any additional information, a purely
impossible task. Actually, in general, it is obviously impossible. However,
under some assumption on the structure of the matrix to be recovered, it
might become feasible, as shown by~\cite{CandesT10} and~\cite{CandesRecht}
where the assumption is that the matrix has a small rank. This assumption
is indeed very natural in many applications. For example, in recommendation
systems, it is equivalent to the existence of a small number of hidden features
that explain the preferences of users.
While~\cite{CandesT10} and~\cite{CandesRecht} focused on matrix completion
without noise, many authors extended these techniques to the
case of noisy observations, see~\cite{CandesP10} and
~\cite{chatterjee2015} among others. The main 
idea in~\cite{CandesP10} is to minimize the least squares criterion,
penalized by the rank. This penalization is then relaxed by the nuclear
norm, which is the sum of the singular values of the matrix at hand.
An efficient algorithm is described in~\cite{Recht2013}.

All the aforementioned papers focused on real-valued matrices. However,
in many applications, the matrix entries are binary, that is in the set 
$\{0,1\}$.
For example, in collaborative filtering, the $(i,j)-th$ entry being $1$ means
that user $i$
likes object $j$ while this entry being $0$ means that he/she dislikes it.
The problem
of recovering a binary matrix from partial observations is usually referred
as 1-bit matrix completion.
To deal with binary observations requires specific estimation methods.
Most works on this problem usually assume a generalized linear model (GLM):
the observations $Y_{ij}$ for $1\leq i\leq m_1$, $1\leq j\leq m_2$,
are Bernoulli distributed with parameter $f(M_{ij})$, where $f$ is a link 
function which maps from $\mathbb{R}$ to $[0,1]$, for example the logistic 
function
$f(x) = \exp(x)/[1+\exp(x)]$, and $M$ is a $m_1 \times m_2$ real matrix, see~\cite{Cai2013,Davenport14,Klopp2015}.
In these works, the goal is to recover 
the matrix $M$ and a convergence rate is then derived. For
example,~\cite{Klopp2015} provides an estimate
$\widehat{M}$ for which, under suitable assumptions and when the data are generated according to the true model with $M=M_0$,
\[
 \frac{1}{m_1 m_2} \|\widehat{M}-M_0\|_{F}^2
 \leq C \max\left(\frac{1}{\sqrt{n}},\frac{(m_1 \vee m_2) {\rm 
rank(M_0)}}{n}\right)
\]
for some constant $C$ that depends on the assumptions, and where $\|.\|_F$ 
stands
for the Frobenius norm (we refer the reader to Corollary 2 page 1955
in~\cite{Klopp2015} for the exact statement). While this result ensures
the consistency of $\widehat{M}$ when $M_0$ is low-rank, it does not provide any
guarantee on the probability of a prediction error. Moreover, such a prediction there would necessarily
assume that the model is correct, that is, that the link function $f$ is well
specified, which is an unrealistic assumption.

Here, we adopt a machine learning point of view: in machine learning, dealing 
with binary output is called a classification problem, for which
methods are known that do not assume any model on the observations.
That is, instead of focusing on a parametric model for $Y_{i,j}$, we will
only define a set of prediction matrices $M$ and seek for the one that
leads to the best prediction error. Using the 0-1 loss function, we could
actually directly use Vapnik-Chervonenkis theory~\cite{VC} to propose a
classifier $\hat{M}$  risk would be controled by a PAC inequality.
However, it is
known that this approach usually is computationally intractable.
A popular approach is to replace the 0-1 loss by a
convex surrogate~\cite{zhang2004}, namely, the hinge loss.
Our approach is as follows: we propose
a pseudo-Bayesian approach, where we define a pseudo-posterior distribution
on a set of matrices $M$. This pseudo-posterior distribution does not have
a very simple form, however, thanks to a variational approximation, we manage to
approximate it by a tractable distribution. Thanks to the PAC-Bayesian
theory~\cite{McA,erbrich:graepel:2002,shawe2003pac,catoni2004statistical,
Catoni2007,seldin2012pac,Dalalyan2008},
we are able to provide a PAC bound on the prediction risk of this
variational approximation. We then show that, thanks to the convex relaxation
of the 0-1 loss, the computation of this
variational approximation is actually a bi-convex minimization problem.
As a consequence, algorithms that are efficient in practice are available.

The rest of the paper is as follows. In Section~\ref{sec:Model} we provide
the notations used in the paper, the definition of the pseudo-posterior and
of its variational approximation. In Section~\ref{sec:pac_bounds} we give the
PAC analysis of the variational approximation. This allows an empirical and a
theoretical upper bound on the prevision risk of our method.
Section~\ref{sec:algorithm} provides details on the implementation of our
method. Note that in the aforementioned sections, the convex surrogate of
the 0-1 loss used is the hinge loss.
An extension to the logistic loss is briefly discussed in
Section~\ref{sec:logistic} and the derived algorithm to compute the variational approximation. Finally, Section~\ref{sec:empri_resu}
is devoted to an empirical study and Section~\ref{sec:movielens}
to an application to the MovieLens data set. The proof of the theorems
of Section~\ref{sec:pac_bounds} are provided in Section~\ref{sec:proofs}.

\section{Notations}
\label{sec:Model}

For any integer $m$ we define $[m]=\{1,\dots,m\}$.
We define, for any integers $m$ and $k$ and any matrix $M
\in \mathbb{R}^{m\times k}$, $\Vert M 
\Vert_\textrm{max} = \max_{(i,j)\in [m]\times [k]} M_{ij}$. For a pair of 
matrices $(M,N)$, we write $\ell(M,N)= \Vert M \Vert_\textrm{max} \vee \Vert N 
\Vert_\textrm{max}$. Finally, when an $m\times k$ matrix $M$ has rank
${\rm rank}(M)=\ell$ then it can be written as $M=UV^T$ where $U$ is $m\times 
\ell$
and $V$ is $k \times \ell$. This decomposition is obviously not unique; we put
$\ell(M) = \inf_{(U,V)}\ell(U,V) $ where the infimum is taken over all such
possible pairs $(U,V)$.

\subsection{1-bit matrix completion as a classification problem}

We formally describe the 1-bit matrix completion problem as a 
classification
problem: we observe $(X_k,Y_k)_{k\in[n]}$ that are $n$ i.i.d pairs
from a distribution $\mathbf{P}$.
The $X_k$'s take values in $\mathcal{X}=[m_1]\times [m_2]$ and
the $Y_k$'s take values in $\mathcal{Y}=\{-1,+1\}$. Hence, the $k$-th 
observation of
an entry of the matrix is $Y_k$ and the corresponding position in the
matrix is provided by $X_k=(i_k,j_k)$. In this setting, a predictor is a 
function
$[m_1]\times [m_2]\rightarrow\mathbb{R}$ and thus can be represented by a matrix
$M$. It would be natural to use $M$ in the following way: when
$(X,Y)\sim \mathbf{P}$, $M$ predicts $Y$ by ${\rm sign}(M_X)$. The ability
of this predictor to predict a new entry of the matrix is then assessed
by the risk
\begin{align*}
R(M) = \mathbb{E}_\mathbf{P}\left[ \mathbbm{1}(Y M_X<0)\right],
\end{align*}
and its empirical counterpart is:
\begin{align*}
r_n(M) = \frac{1}{n}\sum_{k=1}^n \mathbbm{1}(Y_{k} M_{X_k}<0)
 = \frac{1}{n}\sum_{k=1}^n \mathbbm{1}(Y_{k} M_{i_k,j_k}<0) .
\end{align*}
It is then possible to use the standard approach in classification
theory~\cite{VC}.
For example, the best possible classifier is the Bayes classifier
\begin{align*}
\eta(x) = \mathbb{E}(Y|X=x) \quad \textrm{or equivalently} \quad \eta(i,j) = 
\mathbb{E}[Y|X=(i,j)],
\end{align*}
and equivalently we have a corresponding optimal matrix
\begin{align*}
M^B_{ij}=\textrm{sign}[\eta(i,j)] 
=\textrm{sign}\Bigl\{\mathbb{E}[Y|X=(i,j)]\Bigr\}.
\end{align*}
We define
$\overline{R} = \inf_M R(M)= R(M^B)$, and $\overline{r_n}=r_n(M^B)$.
Note that, clearly, if two matrices $M^1$ and $M^2$ are such as, for every 
$(i,j)$, 
$\textrm{sign}(M^1_{ij})=\textrm{sign}(M^2_{ij})$ then $R(M^1)=R(M^2)$,
and obviously,
\begin{align*}
\forall M, \forall (i,j)\in[m_1]\times[m_2], \quad 
\textrm{sign}(M_{ij})=M^B_{ij} \quad \Rightarrow \quad 
r_n(M)=\overline{r_n}.
\end{align*}
While the risk $R(M)$ has a clear interpretation, it is well known
that to work with its empirical counterpart $r_n(M)$ usually leads to
intractable problems, as it is non-smooth and non-convex. Hence, it is
quite standard to replace the empirical
risk by a convex surrogate~\cite{zhang2004}. In this paper, we will mainly
deal with the hinge loss, which leads to the following so-called hinge
risk and hinge empirical risk:
\begin{align*}
R^h(M) &= \mathbb{E}_\mathbf{P}\left[ (1-Y M_X)_+\right],\\
r_n^h(M) &= \frac{1}{n}\sum_{k=1}^n (1-Y_k M_{X_k})_+.
\end{align*}
Note that the hinge loss was also used by~\cite{srebro2004} in the 1-bit
matrix completion problem, with a different approach leading to different
algorithms. Moreover, here, we provide an analysis of the rate of convergence
of our method, that is not provided in~\cite{srebro2004}.

In opposite to other matrix completion works, the marginal distribution of $X$ is not an issue and we do not consider an uniform sampling scheme. Another slight difference is that the observations are iid and a same index may be observed several times. Following standard notations in matrix completion,
we can define $\Omega$ as the set of indices of observed entries:
$\Omega=\{X_1,\dots,X_n\}$. We will use in the following the sub-sample of $\left\lbrace 1,\dots,n\right\rbrace$ for a 
specified 
line $i$: $\Omega_{i,\cdot}=\left\lbrace l \in [n]:(i,j_l) \in \Omega \right\rbrace$ and 
the counterpart for a specified column $j$: $\Omega_{\cdot, j}=\left\lbrace 
l \in [n]:(i_l,j) \in \Omega \right\rbrace$.

\subsection{Pseudo-Bayesian estimation}

The Bayesian framework has been used several times for matrix completion~\cite{mai2015,SalMnih2008,LimTeh2007}. A common idea in all these papers 
is to factorize the matrix into two parts in order to define a prior
on low-rank matrices. Every matrix whose rank is $r$ can
be factorized:
\begin{align*}
M=LR^\top, L\in \mathbb{R}^{m_1 \times r}, \quad R \in \mathbb{R}^{m_2 \times 
r}.
\end{align*}
As mentioned in the introduction, this is motivated by the fact that
we expect that the Bayes matrix $M^B$ is low-rank, or at least well approximated
by a low-rank matrix. However, in practice, we do not know what would be the rank
of this matrix. So, we actually write $M=LR^\top$ with $L\in \mathbb{R}^{m_1 
\times
K}$, $R \in \mathbb{R}^{m_2 \times K}$ for some large enough $K$, and then,
seek for adaptation with respect to a possible rank $r\in[K]$ by shrinking
some columns of $L$ and $R$ to $0$.
Specifically,  we define the prior as the following hierarchical model:
\begin{align*}
\forall k \in [K], \quad \gamma_k &\stackrel{iid}\sim \pi^\gamma, \\
\forall (i,j) \in [m_1] \times [m_2], \quad L_{i,\cdot},R_{j,\cdot}|\gamma 
&\stackrel{iid}\sim \mathcal{N}(0,\textrm{diag}(\gamma)), \\
\text{ and } M& =LR^\top,
\end{align*}
where the prior distribution on the variances $\pi^\gamma$ is yet to be 
specified. Usually $\pi^\gamma$ is an inverse-Gamma distribution because it is conjugate 
in this model. This kind of
hierarchical prior distribution is also very similar to 
the Bayesian Lasso developed in \cite{BayesLasso} and especially of the form 
of 
the Bayesian Group Lasso developed in \cite{BayesGroupLasso} in which the variance term is Gamma distributed. 
We will show that the Gamma distribution is a possible alternative in matrix completion,
both for theoretical results and practical considerations.
Thus all the results in this paper are stated under the assumption that
$\pi^\gamma$ is either the
Gamma or the inverse-Gamma distribution: $\pi^\gamma=\Gamma(\alpha,\beta)$,
or $\pi^\gamma=\Gamma^{-1}(\alpha,\beta)$.

Let $\theta$ denote the parameter $\theta=(L,R,\gamma)$.
As in PAC-Bayes theory~\cite{Catoni2007},
we define the pseudo-posterior as follows:
\begin{align*}
\widehat{\rho}_\lambda(d\theta) = \frac{\exp[-\lambda r_n^h(LR^\top)]}{\int 
\exp[-\lambda r_n^h] d\pi}\pi(d\theta)
\end{align*}
where $\lambda>0$ is a parameter to be fixed by the statistician. The 
calibration
of $\lambda$ is discussed below.
This distribution is close to a classic posterior distribution but the 
log-likelihood has been replaced by the logarithm of the pseudo-likelihood
$r_n^h(LR^\top)$
based on the hinge empirical risk. 

\subsection{Variational Bayes approximations}

Unfortunately, the pseudo-posterior is intractable and MCMC methods are too 
expensive because of the dimension of the parameter. A possible way in order to get an estimate 
is 
to seek an approximation of this distribution. A specific technique, known as 
Variational
Bayes (VB) approximation, allows to replace MCMC methods by efficient
optimization algorithms~\cite{Bishop6}.
First, we fix a subset $\mathcal{F}$ of the set of all distributions on
the parameter space. The class $\mathcal{F}$ should 
be large enough, in order to contain a good enough approximation of
$\widehat{\rho}_\lambda$, but not too large in order to lead to tractable
optimization problems. We usually define the VB approximation as
$\arg\min_{\rho \in 
\mathcal{F}} \mathcal{K}(\rho,\widehat{\rho})$.
However, when $\mathcal{K}(\rho,\widehat{\rho})$ is not available in close
form, it is usual to replace it by an upper bound.
We define here the class $\mathcal{F}$ as follows:
\begin{align*}
\mathcal{F} = \Biggl\lbrace
\rho(d(L,R,\gamma)) & = \prod_{k=1}^{K} \left[ 
\prod_{i=1}^{m_1} \varphi(L_{i,k};L_{i,k}^0,v_{i,k}^L) \prod_{j=1}^{m_2} 
\varphi(R_{j,k};R_{j,k}^0,v_{j,k}^R) \rho^{\gamma_k}(d\gamma_k)\right],
\\
& L^0\in\mathbb{R}^{m_1\times K}, R^0\in\mathbb{R}^{m_2\times K},
v^L \in\mathbb{R}_+^{m_1\times K},v^R \in\mathbb{R}_+^{m_2\times K}
\Biggr\rbrace,
\end{align*}
where $\varphi(.,\mu,v)$ is the density of the Gaussian distribution with 
parameters $(\mu,v)$ and $\rho^{\gamma_k}$ ranges over all possible probability
distributions
for $\gamma^k\in\mathbb{R}^+$. Note that VB approximations are referred
as {\it parametric} when $\mathcal{F}$ is finite dimensional and as
{\it mean-field} otherwise, then we actually use a mixed approach.
Informally, all the coordinates are independent and the 
variational distribution of the entries of $L$ and $R$ is specified. The free 
variational parameters to be optimized are the means and the variances, which 
can 
be both seen in a matrix form. We will show below that the optimization with
respect to $\rho^{\gamma_k}$ is available in close form. Also, note that any
probability distribution $\rho \in \mathcal{F}$ is uniquely determined by
$L^0$, $R^0$, $v^L$, $v^R$ and $\rho^{\gamma_1},\dots,\rho^{\gamma_k}$. We could
actually use the notation $\rho = \rho_{L^0,R^0,v^L,v^R,\rho^{\gamma_1},
\dots,\rho^{\gamma_k}}$, but it would be too cumbersome, so we will avoid
it as much as possible. Conversely, once $\rho$ is given in $\mathcal{F}$, we can define
$L^0 = \mathbb{E}_\rho[L]$, $R^0 = \mathbb{E}_\rho[R]$ and so on.

The Kullback divergence here decomposes as
\begin{align}
\label{eq:Kullback}
\mathcal{K}(\rho,\widehat{\rho}_\lambda) = \lambda \int r_n^h d\rho + 
\mathcal{K}(\rho,\pi) + \log \int \exp [-\lambda r_n^h] d\pi
\end{align}
for which we do not have a close form, so we rather minimize an upper bound.
We will see 
in next sections that this estimate has actually very good properties.

\begin{definition}
Let $\rho=\rho_{L^0,R^0,v^L,v^R,\rho^{\gamma_1},
\dots,\rho^{\gamma_k}}\in\mathcal{F}$ and
 \begin{align*}
\mathcal{R}(\rho,\lambda)
& = \frac{1}{n}
\sum_{\ell=1}^n
\sum_{k=1}^K 
\left[ \sqrt{v^L_{i_\ell,k}\frac{2}{\pi}}\sqrt{v^R_{j_\ell,k}\frac{2}{\pi}} + 
|R_{j_\ell,k}^0|\sqrt{v^L_{i_\ell,k}\frac{2}{\pi}}+ |L_{i_\ell,k}^0|
\sqrt{v^R_{j_\ell,k}\frac{2}{\pi}}  
\right]   \\
& + \frac{1}{\lambda} \left( \frac{1}{2}\sum_{k=1}^K   \mathbb{E}_{\rho}\left[ 
\frac{1}{\gamma_k} \right] \left( \sum_{i=1}^{m_1} (v^L_{ik}+L^{02}_{ik}) +
\sum_{j=1}^{m_2}
(v^R_{jk}+R^{02}_{jk})  \right)- \frac{1}{2}\sum_{k=1}^K \left( \sum_{i=1}^{m_1}
\log v^L_{ik} + 
\sum_{j=1}^{m_2}  \log v^R_{jk}\right)  \right.\\
& \left. + \sum_{k=1}^K  \left[ \mathcal{K}(\rho^{\gamma_k},\pi^\gamma) + 
\frac{m_1+m_2}{2}\left(\mathbb{E}_\rho\left[ \log \gamma_k \right]-1\right) 
\right]  \right).
\end{align*}
\end{definition}

\begin{proposition}
\label{prop:VB_bound}
For any $\rho$ in $\mathcal{F}$,
\begin{equation}
\label{eq:approximate_bound}
\int r_n^h d\rho + \frac{1}{\lambda}\mathcal{K}(\rho,\pi)
\leq r_n^h\left( \mathbb{E}_\rho[L] \mathbb{E}_\rho[R]^\top \right)+ 
\mathcal{R}(\rho,\lambda).
\end{equation}
\end{proposition}
We remind the reader that all the proofs are in Section~\ref{sec:proofs}.
The quantity 
$r_n^h\left( \mathbb{E}_\rho[L] \mathbb{E}_\rho[R]^\top \right)+
\mathcal{R}(\rho,\lambda)$ will be referred to
as the 
Approximate Variational Bound (AVB) in the following. We are now able to define
our estimate.
\begin{definition}
We put
\begin{align}
\nonumber
AVB(\rho) & = r_n^h\left( \mathbb{E}_\rho[L] \mathbb{E}_\rho[R]^\top \right)+
\mathcal{R}(\rho,\lambda),
\\
\widetilde{\rho}_\lambda &= \arg\min_{\rho \in \mathcal{F}}
AVB(\rho).
\label{eq:upperbound}
\end{align}
Also, for simplicity, given $L^0,R^0,v^L,v^R,\rho^{\gamma_1},
\dots,\rho^{\gamma_k}$ we will use the notation
\begin{equation*}
AVB(L^0,R^0,v^L,v^R,\rho^{\gamma_1},
\dots,\rho^{\gamma_k})=AVB(\rho_{L^0,R^0,v^L,v^R,\rho^{\gamma_1},
\dots,\rho^{\gamma_k}}).
\end{equation*}
\end{definition}

\section{PAC analysis of the variational approximation}
\label{sec:pac_bounds}

Paper~\cite{AlquierRidgway2015} proposes a general framework for analyzing
the prediction properties of VB approximations of pseudo-posteriors
based on PAC-Bayesian bounds. In this section, we apply this method to
derive a control of the out-of-sample prevision risk $R$ for our
approximation $\widetilde{\rho}_\lambda$.

\subsection{Empirical Bound}
\label{sec:empirical_bound}

The first result is a so-called empirical bound: it provides an upper bound
on the prevision risk of the pseudo-posterior $\widetilde{\rho}_\lambda$
that depends only on the data and on quantities defined by the statistician.

\begin{theo}\label{th:empirical_bound}
For any $\epsilon\in(0,1)$, with probability at least $1-\epsilon$
on the drawing of the sample,
\begin{align*}
\int R d\widetilde{\rho}_\lambda \leq \inf_{\rho \in \mathcal{F}} \left[ 
r_n^h\left( \mathbb{E}_\rho[L] \mathbb{E}_\rho[R]^\top \right) + 
\mathcal{R}(\rho,\lambda)\right] + \frac{\lambda}{2n} + \frac{\log 
\frac{1}{\epsilon}}{\lambda}
\end{align*}
\end{theo}
We can actually deduce from this result a more explicit bound.
\begin{corollary}
\label{co:empirical_bound}
Assume that $\lambda \leq n$.
For any $\epsilon\in(0,1)$,with probability at least $1-\epsilon$:
\begin{equation*}
\int R d\widetilde{\rho}_\lambda \leq
\inf_{M} \left[ 
r_n^h\left( M \right) + \mathcal{C}_{\pi^\gamma} 
\frac{{\rm rank}(M) (m_1+m_2)[\log 
n+\ell^2 (M) ]}{\lambda}\right] + \frac{\lambda}{2n} + \frac{\log 
\frac{1}{\epsilon}}{\lambda}
\end{equation*}
where the constant $\mathcal{C}_{\pi^\gamma}$ is explicitely known, and
depends only on the prior $\pi^\gamma$ (Gamma, or Inverse-Gamma)
and of the hyperparameters.
\end{corollary}
 An exact value for $\mathcal{C}_{\pi^\gamma}$
can be deduced from the proof.
It is thus clear that the algorithm performs a trade-off between the fit to the 
data, through the term $r_n^h(LM)$, and the rank of
$M$.

\subsection{Theoretical Bound}
\label{sec:theoretical_bound}

For this bound, it is common in classification to make an additional assumption on $\mathbf{P}$ which leads to an easier task and therefore to better rates of convergence. We propose a definition adapted from~\cite{Mammen1999}.

\begin{definition}
Mammen and Tsybakov's margin assumption is satisfied when there is a constant 
$C$ such that: 
\begin{equation*}
\mathbb{E}\left[\left(\mathbbm{1}_{Y M_X\leq 0} - \mathbbm{1}_{Y 
M^B_X\leq 0} \right)^2\right] \leq C[R(M)-\overline{R}].
\end{equation*}
\end{definition}
It is known that it there is a constant $C'>0$ such that
$\mathbb{P}(0<|\eta(X)|\leq t)  \leq C't $
then the margin assumption is satisfied with some $C$ that depends on $C'$.
For example, in the noiseless case where $Y=M^B_X$ almost surely,
then
\begin{equation*}
\mathbb{E}\left[\left(\mathbbm{1}_{Y M_X\leq 0} - \mathbbm{1}_{Y 
M^B_X\leq 0} \right)^2\right]
= \mathbb{E}\left[\mathbbm{1}_{Y M_X\leq 0}^2\right]
= \mathbb{E}\left[\mathbbm{1}_{Y M_X\leq 0}\right] = R(M) = R(M)-\overline{R},
\end{equation*}
so the margin assumption is satisfied with $C=1$.

\begin{theo}
\label{th:theo}
Assume that the 
Mammen and Tsybakov's assumption is satisfied for a given constant $C>0$. Then, 
for any $\epsilon\in(0,1)$ and for $\lambda = s n/C$, $s\in(0,1)$,
with probability at least $1-2\epsilon$,
\begin{align*}
\int R d\widetilde{\rho}_\lambda \leq 2(1+3s)\overline{R} + \mathcal{C}_{C,c,\pi^\gamma} \left( \frac{{\rm rank}(M^B) (m_1+m_2)[\log n+\ell^2(M^B)]+\log\left(\frac{1}{\epsilon}\right)}{n} \right)
\end{align*}
where $\mathcal{C}_{C,s,\pi^\gamma}$ is known and depends only on the constants $s,C$ and the choice of the prior on $\gamma$.
\end{theo} 
Note the adaptive nature of this result, in the sense that the estimator does
{\it not} depend on ${\rm rank}(M^B)$.
\begin{corollary}
\label{co:theo}
In the noiseless case $Y={\rm sign}(M^B_X)$ a.s.,
for any $\epsilon >0$ and for $\lambda = 2 n$,
with probability at least $1-2\epsilon$,
\begin{align}
\int R d\widetilde{\rho}_\lambda \leq 
\mathcal{C}'_{\pi^\gamma}\left[\frac{{\rm rank}(M^B)(m_1+m_2)[\log 
n+\ell^2(M^B)] + \log 
\frac{1}{\epsilon}}{n}\right]
\end{align}
where $\mathcal{C}'_{\pi^\gamma}=\mathcal{C}_{1,\frac{1}{4},1,\pi^\gamma}$.
\end{corollary}

\begin{rmk}
Note that an empirical inequality roughly similar to
Corollary~\ref{co:empirical_bound} appears in~\cite{srebro2004}.
However, in this paper, no oracle inequality similar to
Corollary~\ref{co:theo} is derived - and, due to a slight
modification in their definition of the empirical hinge risk,
we believe that it would not be easy to derive such a bound
from their techniques.
\end{rmk}

\section{Algorithm}
\label{sec:algorithm}

\subsection{General Algorithm}

Note that the minimization problem~\eqref{eq:upperbound} that defines our VB
approximation is not an easy one:
\begin{equation*}
 \min_{L^0,R^0,v^L,v^R,\rho^{\gamma_1},
\dots,\rho^{\gamma_k}} AVB(L^0,R^0,v^L,v^R,\rho^{\gamma_1},
\dots,\rho^{\gamma_k}).
\end{equation*}
When $v^L$, $v^R$ and all the $\rho^{\gamma_k}$'s
are fixed, this is actually the canonical example of so-called biconvex 
problems:
it is convex with respect to $L^0$, and with respect to $R^0$, but not
with respect to the pair $(L^0,R^0)$. Such problems are notoriously
difficult.
In this case, alternating blockwise optimization seems to be an acceptable
strategy. While there is no guarantee that the algorithm will not get stuck
in a local minimum (or even in a singular point that is actually not a minimum),
it seems to give good results in practice, and no efficient alternative is
available. See the discussion Subsection 9.2 page 76 
in~\cite{boyd2011distributed}
for more details on this problem.
We update iteratively $L^0,R^0,v^L,v^R,\rho^{\gamma_1},
\dots,\rho^{\gamma_k}$: for $L^0$ and $R^0$ we use a gradient step,
while for $v^L,v^R,\rho^{\gamma_1},\dots,\rho^{\gamma_k}$ an explicit
minimization is available. The details for the mean-field optimization
(that is, w.r.t. $\rho^{\gamma_1},\dots,\rho^{\gamma_k}$) are given
in Subsection~\ref{sec:MeanFieldOpti}. See Algorithm~\ref{algo:algoHL}
for the general version of the algorithm.
\begin{algorithm}
\caption{Variational Approximation with Hinge Loss}
\label{algo:algoHL}
\begin{algorithmic}
\REQUIRE $\epsilon, (\eta_t)_{t \in \mathbb{N}}, L^0_0, R^0_0, v^L_0,v^R_0, 
\rho^\gamma_0$
\STATE $t \leftarrow 0$
\WHILE{$\Vert L_t^0 R_t^{0\top}-L_{t-1}^0 R_{t-1}^{0\top}\Vert_F^2 \leq 
\epsilon$}
	\STATE $t \leftarrow t+1$
	\STATE $L^0_t \leftarrow L^0_{t-1} - \eta_t \frac{\partial 
AVB}{\partial 
L^0}(L^0_{t-1},R^0_{t-1},v^L_{t-1}, v^R_{t-1}, 
\rho^{\gamma_1}_{t-1},\dots,\rho^{\gamma_K}_{t-1}) $
	\STATE $R^0_t \leftarrow R^0_{t-1} - \eta_t \frac{\partial 
AVB}{\partial 
R^0}(L^0_{t},R^0_{t-1},v^L_{t-1}, v^R_{t-1}, 
\rho^{\gamma_1}_{t-1},\dots,\rho^{\gamma_K}_{t-1}) $
	\STATE $v^L_t \leftarrow \textrm{arg}\min_{v^L} 
AVB(L^0_{t},R^0_{t},v^L, 
v^R_{t-1}, \rho^{\gamma_1}_{t-1},\dots,\rho^{\gamma_K}_{t-1})$
	\STATE $v^R_t \leftarrow \textrm{arg}\min_{v^R} 
AVB(L^0_{t},R^0_{t},v^L_t, v^R, 
\rho^{\gamma_1}_{t-1},\dots,\rho^{\gamma_K}_{t-1})$
	\STATE $(\rho^{\gamma_1}_{t},\dots,\rho^{\gamma_K}_{t}) \leftarrow 
\textrm{arg}\min_{\rho^{\gamma_1},\dots,\rho^{\gamma_K}}AVB(L^0_{t},R^0_{t},v^L_t, v^R, 
\rho^{\gamma_1},\dots,\rho^{\gamma_K})$
\ENDWHILE
\end{algorithmic}
\end{algorithm}

\subsection{Mean Field Optimization}
\label{sec:MeanFieldOpti}

Note that the
pseudo-likelihood does not involve the parameters $(\gamma_1,\dots,\gamma_K)$
so the variational 
distribution can be optimized in the same way as the model in 
\cite{LimTeh2007} 
where the noise is Gaussian. The general update formula is ($\rho^{-\gamma_k}$ stands for the whole distribution $\rho$ except the part involving $\gamma_k$):
\begin{align*}
\rho^{\gamma_k}(\gamma_k) &\propto  \exp \mathbb{E}_{\rho^{-\gamma_k}} 
\left(\log \pi(L,R,\gamma)  \right)\\
& \propto \exp \mathbb{E}_{\rho^{-\gamma_k}} \left( \log 
[\pi(L|\gamma)\pi(R|\gamma)\pi^\gamma(\gamma)] \right)  \\
& \propto \exp \left\lbrace \sum_{i=1}^{m_1} \mathbb{E}_{\rho^L} \left[ \log 
\pi(L_{i,k}|\gamma_k) \right] + \sum_{j=1}^{m_2}\mathbb{E}_{\rho^R} \left[ \log 
\pi(R_{j,k}|\gamma_k) \right]+ \log \pi^\gamma \left(\gamma_k \right) \right\rbrace\\
& \propto \exp \left\lbrace -\frac{m_1+m_2}{2} \log \gamma_k - 
\frac{1}{\gamma_k} \mathbb{E}_\rho \left[ \frac{ \sum_{i=1}^{m_1} L_{i,k}^2 + 
\sum_{j=1}^{m_2} R_{j,k}^2}{2}\right] + \log \pi^\gamma \left(\gamma_k \right)
\right\rbrace .
\end{align*}
The solution then depends on $\pi^\gamma$. In what follows we derive explicit
formulas for $\rho^{\gamma_k}$ according to the choice of $\pi^\gamma$: remember that
$\pi^\gamma$ could be either a Gamma distribution, or an Inverse-Gamma distribution.

\subsubsection{Inverse-Gamma Prior}

The conjugate prior for 
this part of the model is the inverse-Gamma distribution. The prior of 
$\gamma_k$ is $\pi^{\gamma}= \Gamma^{-1}(\alpha,\beta)$ and its density is:
\begin{align*}
\pi^\gamma(\gamma_k;\alpha,\beta)=\frac{\beta^\alpha}{\Gamma(\alpha)}\gamma_k^{
-\alpha-1}\exp \left(-\frac{\beta}{\gamma_k} \right) 
\mathbbm{1}_{\mathbb{R}^+}(\gamma_k).
\end{align*}
The moments we need to develop the algorithm and to compute the empirical bound 
are:
\begin{align*}
\mathbb{E}_{\pi^\gamma}(\log \gamma_k) = \log \beta - \psi(\alpha)\text{, and }
\mathbb{E}_{\pi^\gamma}(1/\gamma_k) = \frac{\alpha}{\beta},
\end{align*}
where $\psi$ is the digamma function. Therefore, we get:
\begin{align*}
\rho^{\gamma_k}(\gamma_k) & \propto  \exp \left\lbrace 
-\left(\frac{m_1+m_2}{2}+\alpha+1\right)\log \gamma^k- 
\frac{1}{\gamma_k}\left(\mathbb{E}_\rho \left[ \frac{ \sum_{i=1}^{m_1} L_{i,k}^2 
+ \sum_{j=1}^{m_2} R_{j,k}^2}{2}\right] +\beta\right) \right\rbrace, \\
\end{align*}
so we can conclude that:
\begin{align}
\rho^{\gamma_k}=\Gamma^{-1}\left(\frac{m_1+m_2}{2}+ \alpha,\mathbb{E}_\rho \left[ 
\frac{ \sum_{i=1}^{m_1} L_{i,k}^2 + \sum_{j=1}^{m_2} R_{j,k}^2}{2}\right]  +\beta 
\right).
\end{align} 

\subsubsection{Gamma Prior}

Even though it seems that this fact was not used in prior works on Bayesian
matrix estimation, it is also possible to derive explicit formulas when
the prior $\pi^\gamma$ on $\gamma_k$'s is a $\Gamma(\alpha,\beta)$
distribution.
In this case, $\rho^{\gamma_k}$ is given by 
\begin{align*}
\rho^{\gamma_k}(\gamma_k) &\propto \exp \left\lbrace 
\left(\alpha-\frac{m_1+m_2}{2}-1\right)\log \gamma_k - \beta \gamma_k 
-\frac{1}{\gamma_k}\mathbb{E}_\rho \left[ \frac{ \sum_{i=1}^{m_1} L_{i,k}^2 + 
\sum_{j=1}^{m_2} R_{j,k}^2}{2}\right] \right\rbrace .
\end{align*}
We remind the reader that the
Generalized Inverse Gaussian distribution is a three-parameter family
of distributions over $\mathbb{R}^{+*}$,
written $GIG(a, b,\eta)$. Its density is given by:
\begin{align*}
f(x;a,b,\eta) &= \frac{(a / b)^{\eta/2}}{2K_\eta(\sqrt{a b})} 
x^{\eta-1}\exp\left( -\frac{1}{2}(a x + b x^{-1}) \right),
\end{align*}
where $K_\lambda$ is the modified Bessel function 
of second kind.

The variational distribution $\rho^{\gamma_k}$ is in consequence 
$GIG(a_k,b_k,\eta_k)$ with:
\begin{align*}
a_k &= 2\beta , \quad b_k = \mathbb{E}_\rho \left[ \frac{ \sum_{i=1}^{m_1} 
L_{i,k}^2 + \sum_{j=1}^{m_2} R_{j,k}^2}{2}\right], \quad \eta_k 
=\alpha-\frac{m_1+m_2}{2}.
\end{align*}
The moment we need in order to compute the variational distribution of $L,R$ is:
\begin{align*}
\mathbb{E}_{\rho^{\gamma_k}}\left( \frac{1}{\gamma_k} \right) = 
\frac{K_{\eta_k-1}(\sqrt{a_k b_k})}{K_{\eta_k}(\sqrt{a_k b_k})} 
\sqrt{\frac{a_k}{b_k}}.
\end{align*}

%%%%%%%%%%%%%%%%%%%%%%%%%%%%%%%%%%%%%%%%%%%%%%%%%%%%%%%%%%%%%%%%%%%%%
%%%%%%%%%%%%%%%%%%%%%%%%%%%%%%%%%%%%%%%%%%%%%%%%%%%%%%%%%%%%%%%%%%%%%
%%%%%%% Logistic Model %%%%%%%%%%%%%%%%%%%%%%%%%%%%%%%%%%%%%%%%%%%%%%
%%%%%%%%%%%%%%%%%%%%%%%%%%%%%%%%%%%%%%%%%%%%%%%%%%%%%%%%%%%%%%%%%%%%%

\section{Logistic Model}
\label{sec:logistic}

As mentioned in~\cite{zhang2004},
the hinge loss is not the only possible convex relaxation
of the 0-1 loss. The logistic loss
${\rm logit}(u)=\log[1+\exp(-u)]$
can also be used (even though it might lead to a loss in the rate
of convergence). This leads
to the risks:
\begin{align*}
R^\ell(M) &= \mathbb{E}_\mathbf{P}\left[ {\rm logit}(Y M_X) \right],\\
r_n^\ell(M) &= \frac{1}{n}\sum_{k=1}^n {\rm logit}(Y_k M_{X_k}).
\end{align*}
Note that then the pseudo-likelihood $\exp(-\lambda r_n^\ell(M))$ becomes
exactly equal to the likelihood if $\lambda=n$ and we assume a logistic model,
that is $Y_k = 2y_k-1$,
$y_k|X_k\in \mathcal{B}e( \sigma(M_{X_k}) )$ where $\sigma$ is the link function
$\sigma(x) = \frac{\exp(x)}{1+\exp(x)}$. For the sake of coherence with previous sections, we still use the machine learning notations and the likelihood is written $\Lambda(L,R)=\prod_{l=1}^{n}\sigma(Y_l (LR^\top)_{X_l})$. The prior distribution is exactly the same and the object of interest is the posterior distribution:
\begin{align*}
\widehat{\rho_l}(d\theta)=\frac{\Lambda(L,R)\pi(d\theta)}{\int \Lambda(L,R)\pi(d\theta)}.
\end{align*}

In order to deal with large matrices, it is still interesting to develop a variational Bayes algorithm. However it is not as simple as in the quadratic loss model, see \cite{LimTeh2007} in which the authors develop a mean field approximation, because the logistic likelihood leads to intractable update formulas. A common way to deal with this model is to maximize another quantity which is very close to the one we are interested in. The principle, coming from \cite{jaakkola2000}, is well explained in \cite{Bishop6} and an extended example can be found in \cite{latouche2015}. 

We consider the mean field approximation so the approximation is sought among the distributions $\rho$ that are factorized $\rho(d\theta) = \prod_{i=1}^{m_1}\rho^{L_i}(dL_{i,\cdot}) 
\prod_{j=1}^{m_2}\rho^{R_j}(dR_{j,\cdot}) \prod_{k=1}^{K}\rho^{\gamma_k}(d\gamma_k)$. We have the 
following decomposition, for all distribution $\rho$:
\begin{align*}
\log \int \Lambda(L,R)\pi(d\theta) &= \mathcal{L}(\rho) + \mathcal{K}(\rho,\widehat{\rho_l}) \\
\textrm{with } \mathcal{L}(\rho) &= \int \log \frac{\Lambda(L,R)\pi(\theta)}{\rho(\theta)}\rho(d\theta) .
\end{align*}

Since the log-evidence is fixed, minimizing the Kullback divergence w.r.t. $\rho$ is the same as maximizing $\mathcal{L}(\rho)$. Unfortunately, this quantity is intractable but a lower bound, which corresponds to a Gaussian approximation, is much more easier to optimize. We introduce the additional parameter $\xi = (\xi_l)_{l \in [n]}$.

\begin{proposition}
\label{prop:VariationalLaplace}
For all $\xi \in \mathbb{R}^n$ and for all $\rho$,
\begin{align*}
\mathcal{L}(\rho) & \geq \int \log 
\frac{H(\theta,\xi)\pi(\theta)}{\rho(\theta)} \rho(d\theta) := \mathcal{L}(\rho,\xi)\\
\textrm{where } \log H(\theta,\xi) &=\sum_{l \in [n]} \left\lbrace
\log \sigma(\xi_l) + \frac{Y_l (LR^\top)_{X_l} - \xi_l}{2} - \tau(\xi_l)\left[ (LR^\top)_{X_l}^2-\xi_l^2 \right] \right\rbrace
\end{align*}
\end{proposition}

The algorithm is then straightforward: we optimize $\mathcal{L}(\rho,\xi)$ w.r.t $(\rho,\xi)$ and expect that, eventually, the approximation is not too far from $\widehat{\rho_l}$.

\subsection{Variational Bayes Algorithm}

The lower bound $\mathcal{L}(\rho,\xi)$ is maximized with respect to $\rho$ by the 
mean field algorithm. A direct calculation gives that the optimal distribution 
of each site (written with a star subscript) is given by:
\begin{align*}
\forall i \in [m_1], \log \rho_\star^{L_{i,\cdot}}(L_{i,\cdot}) &= \int \log \left[ H(\theta,\xi)\pi(\theta) \right] \rho^R(dR)\rho^\gamma(d\gamma)\prod_{i'\neq i} \rho(dL_{i',\cdot}) + \textrm{const} \\
\forall j \in [m_2], \log \rho_\star^{R_{j,\cdot}}(R_{j,\cdot}) &= \int \log \left[ H(\theta,\xi)\pi(\theta) \right] \rho^L(dL)\rho^\gamma(d\gamma)\prod_{j'\neq j} \rho(dR_{j',\cdot}) + \textrm{const}
\end{align*}

As $\log H(\theta,\xi)$ is a quadratic form in $(L_{i,\cdot})_{i\in [m_1]}$ and $(R_{j,\cdot})_{j\in [m_2]}$, the variational distribution of each parameter is Gaussian and a direct calculation gives: 
\begin{align*}
\rho_\star^{L_{i,\cdot}} &= \mathcal{N}\left(\mathcal{M}_i^L, \mathcal{V}_i^L\right), \quad  \rho_\star^{R_{j,\cdot}} = \mathcal{N}\left(\mathcal{M}_j^R, \mathcal{V}_j^R\right) \quad  \textrm{where} \\
\mathcal{M}_i^L &= \left( \frac{1}{2} \sum_{l \in \Omega_{i,\cdot}}Y_l\mathbb{E}_{\rho} \left[ R_{j_l,\cdot}\right] \right) \mathcal{V}_i^L, \quad 
\mathcal{V}_i^L = \left(2\sum_{l \in \Omega_{i,\cdot}} \tau(\xi_l) \mathbb{E}_\rho[R_{j_l,\cdot}^\top R_{j_l,\cdot}]+ \mathbb{E}_\rho \left[\textrm{diag}\left(\frac{1}{\gamma}\right)\right]\right)^{-1}   \\
\mathcal{M}_j^R &=  \left( \frac{1}{2}\sum_{l \in \Omega_{\cdot,j}} Y_l\mathbb{E}_{q} \left[ L_{i_l, \cdot} \right] \right) \mathcal{V}_j^R, \quad 
\mathcal{V}_j^R = \left(2\sum_{l \in \Omega_{\cdot, j}} \tau(\xi_l) \mathbb{E}_\rho [L_{i_l, \cdot}^\top L_{i_l, \cdot}]+ \mathbb{E}_\rho \left[\textrm{diag}\left(\frac{1}{\gamma} \right)\right]\right)^{-1} 
\end{align*}

The variational optimization for $\gamma$ is exactly the same as in the Hinge 
Loss model (with both possible prior distributions $\Gamma$ and $\Gamma^{-1}$). The 
optimization of the variational parameters is given by:
\begin{align*}
\forall l \in [n], \quad \widehat{\xi_l} &= \sqrt{\mathbb{E}_\rho \left[ (LR^\top)_{X_l}^2 \right]}
\end{align*}

\section{Empirical Results}
\label{sec:empri_resu}

The aim of this section is to compare our methods to the other 1-bit matrix completion techniques. Although the lack of risk bounds for GLM models, we can expect that the performance will not be worse than the one from our estimate in a general setting and will be the target for data generated according to the logistic model. It is worth noting that the low rank decomposition does not involve the same matrix: in our model, it affects the Bayesian classifier matrix; in logistic model, it concerns the parameter matrix. The estimate from our algorithm is $\widehat{M}=\mathbb{E}_{\widetilde{\rho}_\lambda}(L) \mathbb{E}_{\widetilde{\rho}_\lambda}(R)^\top$ and we focus on the zero-one loss in prediction. We first test the performances on simulated matrices and then experiment them on a real data set. We compare the three following models: (a) hinge loss with variational approximation (referred as \emph{HL}), (b) Bayesian logistic model with variational approximation (referred as \emph{Logit}) and (c) the frequentist logistic model from \cite{Davenport14} (referred as \emph{freq. Logit.}). The former two are tested with both Gamma and Inverse-Gamma prior distributions. The hyperparameters are all tuned by cross validation.

\subsection{Simulated Matrices}
The goal is to assess the models with different kind of data generation. The general scheme of the simulations is as follows: the observations come from a $200\times 200$ matrix and we pick randomly $20\%$ of its entries. In our algorithms, we set $K=10$ for computational reason, but it works well with a larger value. The observations are generated as:
\begin{align*}
 Y_l = \textrm{sign}\left( M_{i_l,j_l}+Z_l\right)B_l, \quad \textrm{where } M\in \mathbb{R}^{m\times m}, \quad (B_l,Z_l)_{l\in [n]} \textrm{ are iid}.
\end{align*}

The noise term $(B,Z)$ is such that $R(M)=\overline{R}$ and $M$ has low rank. The predictions are directly compared to $M$. Two types of matrices $M$ are built: the \emph{type A} corresponds to the best case of the hinge loss model and the entries of $M$ lie in $\{-1,+1\}$\footnote{The matrices are built by drawing $r$ independent columns with only $\{-1,1\}$. The remaining columns are randomly equal to one of the first $r$ columns multiplied by a factor in $\{-1,1\}$.}. The \emph{type B} corresponds to the a more difficult classification problem because many entries of $M$ are around $0$: $M$ is a product of two matrices with $r$ columns where the entries are iid $\mathcal{N}(0,1)$. The noise term is specified in Table \ref{tab:type_noise}. Note that the example A3 may also be seen as a switch noise with probability $\frac{e}{1+e} \approx 0.73$. 

\begin{table}[htbp]
\centering
\caption{Type of Noise}
\begin{tabular}{lllll}
Type & Name &  $B$ & $Z$ & $Y$ \\
\hline
$1$ & No noise 	& $B=1$ a.s. & $Z=0$ a.s. & $Y_l=\textrm{sign}(M_{i_l,j_l})$ a.s \\
$2$ & Switch 	& $B\sim 0.9 \delta_1 + 0.1 \delta_{-1}$ & $Z=0$ a.s.  & $Y_l=\textrm{sign}(M_{i_l,j_l})$ w.p. $0.9$ \\
$3$ & Logistic 	& $B=1$ a.s. & $Z \sim \textrm{Logistic}$ & $Y_l =1$ w.p. $\sigma(M_{i_l,j_l})$ 
\end{tabular}
\label{tab:type_noise}
\end{table}

\begin{table}[htbp]
  \centering
  \caption{Results on Simulated Observations - rank $3$}
    \begin{tabular}{lccccc}
    Type & Logit-G & Logit-IG & HL-G & HL-IG & freq. Logit \\
    \hline \hline
    A1    & 0.0\% & 0.0\% & 0.0\% & 0.0\% & 0.0\% \\
    A2    & 0.5\% & 0.9\% & 0.1\% & 0.0\% & 0.5\% \\
    A3    & 16.0\% & 15.9\% & 8.5\% & 8.5\% & 17.3\% \\
    \hline 
    B1    & 4.1\% & 4.0\% & 5.3\% & 5.8\% & 5.1\% \\
    B2    & 10.1\% & 10.1\% & 10.8\% & 10.6\% & 10.7\%\\
    B3    & 16.0\% & 16.0\% & 22.1\% & 21.3\% & 19.8\%\\
    \end{tabular}
  \label{tab:SimuR3}
\end{table}

For rank $3$ (see Table \ref{tab:SimuR3}) and rank $5$ matrices (see Table \ref{tab:SimuR5}), the results of the Bayesian models are very similar for both prior distributions and there is no evidence to favor a particular one. The results are better for the hinge loss model on type A observations and the difference of performance between models is very large for A3. On the opposite, the performance of the logistic model is better when the observations are generated from this model and when the parameter matrix is not separable. In comparison with the results from the frequentist model,  the variational approximation seems very good even though we have not at all any theoretical properties. For rank $5$ matrices, the performances are worse but the meanings are the same as the rank $3$ experiment.

\begin{table}[htbp]
  \centering
   \caption{Results on Simulated Matrices - rank $5$}
    \begin{tabular}{lccccc}
    Type & Logit-G & Logit-IG & HL-G & HL-IG & freq. Logit \\
    \hline \hline
    A1    & 0.01\% & 0.01\% & 0,0\% & 0,0\% & 0.01\% \\
    A2    & 4.4\% & 3.1\% & 0.54\% & 0.55\% & 3.1\% \\
    A3    & 32.5\% & 33.1\% & 27,0\% & 26.7\% & 30.1\% \\
    B1    & 7.8\% & 7.8\% & 9.4\% & 10.4\% & 9.0\% \\
    B2    & 17.3\% & 17.3\% & 17.9\% & 18.1\% & 18.3\% \\
    B3    & 21.5\% & 21.4\% & 24.4\% & 22.9\% & 22.1\% \\
    \end{tabular}%
  \label{tab:SimuR5}
\end{table}

The last experiment is a focus on the influence of the level of switch noise. On $A2$ type on rank $3$, we see that $10\%$ of corrupted entries is not enough to almost perfectly recover the Bayesian classifier matrix. We challenge the frequentist program as well. The results are clear and the hinge loss model is better almost everywhere. For a noise up to $25\%$, which means that one fourth of observed entries are corrupted, it is possible to get a very good predictor with less than $10\%$ of misclassification error. It is getting worse when the level of noise increases and the problem becomes almost impossible for noise greater than $30\%$.

\begin{figure}
\centering
\caption{Results on Simulated  A2 Matrices with different levels of noise - 
rank 
$3$}
\includegraphics[scale=.63]{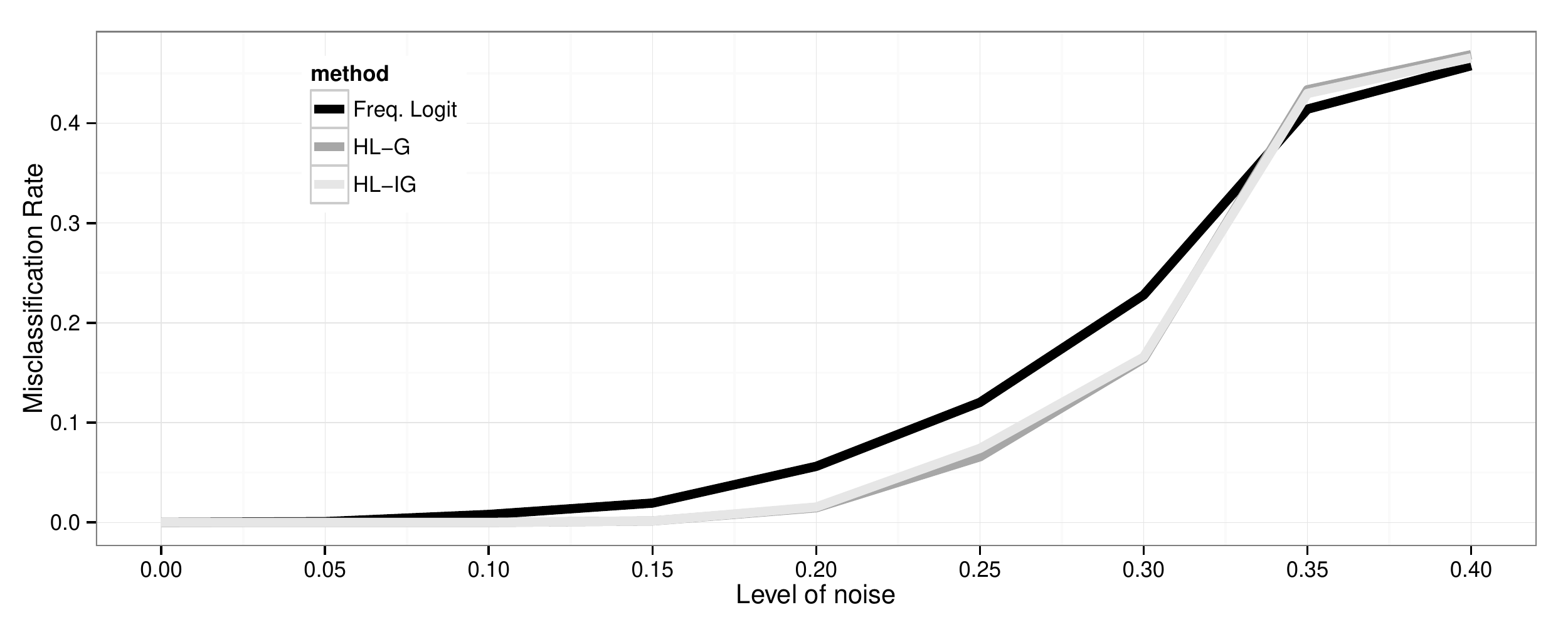}
\label{fig:SimuNoise}
\end{figure}

\subsection{Real Data set: MovieLens}
\label{sec:movielens}

The last experiment involves the well known MovieLens\footnote{Available at http://grouplens.org/datasets/movielens/100k/} data set. It has already been used by \cite{Davenport14} and we follow them for the study. The ratings lie between $1$ to $5$ so we split them into binary data between good ratings (above the mean which is 3.5) and bad ones. The low rank assumption is usual in this case because it is expected that the taste of a particular user is related to only few hidden parameters. The smallest data set contains 100,000 ratings from 943 users and 1682 movies so we use 95,000 of them as a training set and the 5,000 remaining as the test set. The performances are very similar between the frequentist logistic model from \cite{Davenport14} and the hinge loss model. The performances are slightly worse for the Bayesian logistic model so it is hard to favour a particular model at this stage.

\begin{table}[hbt]
\centering
\caption{Performance rate on MovieLens 100k data set}
\begin{tabular}{l|ccccc}
Algorithm 	& HL-IG	& HL-G 	& Logit-G	& Logit-IG	&  Freq. Logit  \\
\hline
correct classif. rate & .72	& .71	&	.68		&		.68	&	.73		
\end{tabular}
\label{tab:MovieLens}
\end{table}

\section{Discussion}

We approach the 1-bit matrix completion problem with classification tools
and it allows us to derive PAC-bounds on the risk. The previous works only
focused on GLM models, which is not the right way to establish such risk bounds.
This work relies on PAC-Bayesian framework and the pseudo-posterior distribution
is approximated by a variational algorithm. In practice, it is able to deal with
large matrices. We also derive a variational approximation of the posterior distribution in the Bayesian logistic model and it works very well in our examples.

The variational approximations look very promising in order to build algorithm which are able to deal with large data and this framework may be extended to more general models and other Machine Learning tools. 

\section*{Acknowledgment}
We would like to thank Vincent Cottet's PhD supervisor Professor
Nicolas Chopin, for his useful support during the project.

%%%%%%%%%%%%%%%%%%%%%%%%%%%%%%%%%%%%%%%%%%%%%%%%%%%%%%%%%%%%%%%%%%%%%
%%%%%%%%%%%%%%%%%%%%%%%%%%%%%%%%%%%%%%%%%%%%%%%%%%%%%%%%%%%%%%%%%%%%%
%%%%%%% Proofs         %%%%%%%%%%%%%%%%%%%%%%%%%%%%%%%%%%%%%%%%%%%%%%
%%%%%%%%%%%%%%%%%%%%%%%%%%%%%%%%%%%%%%%%%%%%%%%%%%%%%%%%%%%%%%%%%%%%%

\section{Proofs}
\label{sec:proofs}

\subsection{Proofs of Proposition~\ref{prop:VB_bound} from Section~\ref{sec:Model}}
\label{sec:proof_model}
\begin{proof}[Proof of Proposition~\ref{prop:VB_bound}]
Let $\rho=\rho_{L^0,R^0,v^L,v^R,\rho^{\gamma_1},
\dots,\rho^{\gamma_k}}$ be a distribution from $\mathcal{F}$. The first term
in~\eqref{eq:Kullback} is upper bounded by the Lipschitz property of the hinge 
loss:
\begin{align*}
&\int r_n^h(LR^\top)\rho(dL,dR,d\gamma) = \frac{1}{n} \sum_{\ell=1}^n 
\int 
\left(1-Y_{\ell}L_{i_\ell,\cdot} R_{j_\ell,\cdot}^\top \right)_+
\rho(dL,dR,d\gamma)
%&=   \frac{1}{n}  \sum_{i,j \in \Omega} \left(\left(1-Y_{ij}L_i^0 R_j^{0\top} 
%\right)_+ +\int \left[ \left(1-Y_{ij}L_i R_j^\top \right)_+ - 
%\left(1-Y_{ij}L_i^0 R_j^{0\top}\right)_+ \right] \rho(dL,dR,d\gamma) \right)
\\
&\leq  \frac{1}{n} \sum_{\ell=1}^n  \left(\left(1-Y_{\ell}L_{i_\ell,\cdot}^0
R_{j_\ell,\cdot}^{0 \top} \right)_+
+\int |L_{i_\ell,\cdot} R_{j_\ell,\cdot}^\top -L_{i_\ell,\cdot}^0
R_{j_\ell,\cdot}^{0 \top} | \rho(dL,dR)  \right) \\
&\leq r_n^h\left(L^0R^{0\top}\right) +  \frac{1}{n}  \sum_{\ell=1}^n \sum_{k=1}^K
\left[ \sqrt{v^L_{i_\ell,k}\frac{2}{\pi}}\sqrt{v^R_{j_\ell,k}\frac{2}{\pi}} + 
|R_{j\ell,k}^0|\sqrt{v^L_{i_\ell,k}\frac{2}{\pi}}+
|L_{i_\ell,k}^0|\sqrt{v^R_{j_\ell,k}\frac{2}{\pi}}  
\right].
\end{align*}
The second part (KL-divergence) can be explicitly calculated. Let
$\rho^{L_{i,k}}$ denote the marginal distribution of $L_{i,k}$ under $\rho$.
We define in the same way $\rho^{R_{j,k}}$. Also, $\pi\rho^{L_{i,k}|\gamma_k}$
denote the distribution of $L_{i,k}$ given $\gamma_k$ under $\pi$, and
we define in the same way $\pi\rho^{R_{j,k}|\gamma_k}$. Then we have
\begin{align*}
\mathcal{K}(\rho,\pi) &= \sum_{k=1}^K \left[
\sum_{i=1}^{m_1} \mathbb{E}_{\rho^{\gamma_k}}\left[ 
\mathcal{K}(\rho^{L_{i,k}},\pi^{L_{i,k}|\gamma_k}) \right]+ \sum_{j=1}^{m_2}
\mathbb{E}_{\rho^{\gamma_k}}\left[ \mathcal{K}(\rho^{R_{j,k}},
\pi^{R_{j,k}|\gamma_k})\right]
+ \mathcal{K}(\rho^{\gamma_k},\pi^{\gamma_k}) \right] \\
&=\frac{1}{2}\sum_{k=1}^K  \mathbb{E}_{\rho^{\gamma_k}}\left[ \frac{1}{\gamma_k} 
\right] 
\left( \sum_{i=1}^{m_1} (v^L_{i,k}+L^{02}_{i,k})
+\sum_{j=1}^{m_2} (v^R_{jk}+R^{02}_{j,k})  \right)  
\\
& - \frac{1}{2}\sum_{k=1}^K \left( \sum_{i=1}^{m_1} \log v^L_{i,k}
+ \sum_{j=1}^{m_2}\log v^R_{j,k}\right) 
+ 
\sum_{k=1}^K\left[ \mathcal{K}(\rho^{\gamma_k},\pi^{\gamma_k}) + 
\frac{m_1+m_2}{2}\left(\mathbb{E}_{\rho^{\gamma_k}}\left[ \log \gamma_k
\right]-1\right) \right]
\\
 &= \frac{1}{2}\sum_{k=1}^K \left[ \sum_{i=1}^{m_1} 
\left(\mathbb{E}_{\rho^{\gamma_k}}\left[ \frac{1}{\gamma_k} \right] 
(v^L_{i,k}+L^{02}_{i,k})+\mathbb{E}_{\rho^{\gamma_k}}\left[ \log \gamma_k \right] - 
\log v^L_{i,k} -1  \right) \right.\\
& \quad \left.+  \sum_{j=1}^{m_2} \left(\mathbb{E}_{\rho^{\gamma_k}}\left[ 
\frac{1}{\gamma_k} \right] 
(v^R_{j,k}+R^{02}_{j,k})+\mathbb{E}_{\rho^{\gamma_k}}\left[ \log \gamma_k \right] - 
\log v^R_{j,k} -1  \right) + 2\mathcal{K}(\rho^{\gamma_k},\pi^{\gamma_k}) \right].
\end{align*}
\end{proof}

\subsection{Proofs of the results in Subsection~\ref{sec:empirical_bound}}
\label{sec:proof_empiric}

\begin{proof}[Proof of Theorem~\ref{th:empirical_bound}]
As the indicator function is uniformly bounded by $1$, we can use Lemma 5.1
in~\cite{AlquierRidgway2015}:
\begin{align*}
\left. \begin{array}{l}
\int \mathbb{E}\exp\{\lambda [ R(LR^\top)-r_n(LR^\top)]\} {\rm d}\pi(R,L,\gamma) \\
\int \mathbb{E}\exp\{\lambda [ r_n(LR^\top)-R(LR^\top)]\} {\rm d}\pi(R,L,\gamma)
\end{array} \right\rbrace \leq 
\exp\left[\frac{\lambda^2}{2n}\right].
\end{align*}
So, the assumptions of Theorem 4.1 in~\cite{AlquierRidgway2015} are satisfied
and we obtain that, for  any $\epsilon\in(0,1)$,
with probability at least $1-\epsilon$ on the drawing of the sample,
for any $\rho$ in $\mathcal{F}$:
\begin{align*}
\int Rd\rho &\leq \int r_nd\rho + 
\frac{\mathcal{K}(\rho,\pi)}{\lambda}+\frac{\lambda}{2n}+\frac{\log 
\frac{1}{\epsilon}}{\lambda} \\
& \leq \int r_n^h d\rho + 
\frac{\mathcal{K}(\rho,\pi)}{\lambda}+\frac{\lambda}{2n}+\frac{\log 
\frac{1}{\epsilon}}{\lambda}
\text{ (as } r_n^h \geq r_n),
\\
& \leq r_n^h\left( L^0 R^{0\top}\right)+ 
\mathcal{R}(\rho,\lambda) + \frac{\lambda}{2n} + \frac{\log 
\frac{1}{\epsilon}}{\lambda}
\text{ (thanks to Proposition~\ref{prop:VB_bound}).}
\end{align*}
We end the proof by minimizing the right-hand-side w.r.t $\rho\in\mathcal{F}$.
\end{proof}

In order to prove Corollary~\ref{co:empirical_bound}, we need a preliminary
result.
For any $m_1\times m_2$
matrix $M$ with rank $r\in [K]$, we can write $M=LR^T$ where $L$ is
$m_1\times K$, $R$ is
$m_2\times K$ and, up to a reordering of the columns, $L_{\cdot,r+1}=\dots=
L_{\cdot,K}=0$ and $R_{\cdot,r+1}=\dots=R_{\cdot,K}=0$.
We denote by
$\mathcal{B}(M)$ the set of such pairs of matrices $(L,R)$ and
$$
\mathcal{F}(M)
=
\left\{
\rho \in \mathcal{F}: (\mathbb{E}_{\rho}(L),\mathbb{E}_{\rho}(R))
\in \mathcal{B}(M)
\right\}.
$$
\begin{lemma}
\label{lemma:rank}
There is a constant $\mathcal{C}_{\pi^\gamma} $ that depends only on
the choice of the prior $\pi^\gamma$ such that for any $\lambda \leq n$,
\begin{align*}
\inf_{\rho \in \mathcal{F}(M)} \mathcal{R}(\rho,\lambda) \leq 
\mathcal{C}_{\pi^\gamma} \frac{{\rm rank}(M)(m_1+m_2)(\ell(M)^2+\log n)}{\lambda},
\end{align*}
$\mathcal{C}_{\pi^\gamma}$ is explicitly known and
depends only on the choice of the prior $\pi^\gamma$ (Gamma, or Inverse-Gamma)
and of its parameters.
\end{lemma}
It is obvious that when we combine Lemma~\ref{lemma:rank} with
Theorem~\ref{th:empirical_bound} we obtain Corollary~\ref{co:empirical_bound}.
It remains to prove the lemma.
\begin{proof}
Let $M$ be fixed and for short let $r$ denote ${\rm rank}(M)$. We remind that,
by definition:
 \begin{align*}
\mathcal{R}(\rho,\lambda)
& = \frac{1}{n}
\sum_{\ell=1}^n
\sum_{k=1}^K 
\left[ \sqrt{v^L_{i_\ell,k}\frac{2}{\pi}}\sqrt{v^R_{j_\ell,k}\frac{2}{\pi}} + 
|R_{j_\ell,k}^0|\sqrt{v^L_{i_\ell,k}\frac{2}{\pi}}+ |L_{i_\ell,k}^0|
\sqrt{v^R_{j_\ell,k}\frac{2}{\pi}}  
\right]   \\
& + \frac{1}{\lambda} \left( \frac{1}{2}\sum_{k=1}^K   \mathbb{E}_{\rho}\left[ 
\frac{1}{\gamma_k} \right] \left( \sum_{i=1}^{m_1} (v^L_{ik}+L^{02}_{ik}) +
\sum_{j=1}^{m_2}
(v^R_{jk}+R^{02}_{jk})  \right)- \frac{1}{2}\sum_{k=1}^K \left( \sum_{i=1}^{m_1}
\log v^L_{ik} + 
\sum_{j=1}^{m_2}  \log v^R_{jk}\right)  \right.\\
& \left. + \sum_{k=1}^K  \left[ \mathcal{K}(\rho^{\gamma_k},\pi^\gamma) + 
\frac{m_1+m_2}{2}\left(\mathbb{E}_\rho\left[ \log \gamma_k \right]-1\right) 
\right]  \right).
\end{align*}
We will now upper bound the infimum for a special choice for
$\rho=\rho_{L^0,R^0,v^L,v^R,\rho^{\gamma_1},
\dots,\rho^{\gamma_k}}$ with $(L^0,R^0)\in\mathcal{B}(M)$:
for all pairs $(i,k)$ and 
$(j,k')$ $v_{i,k}^L=v_{j,k}^R=v^0$ when $k,k'\leq r$
$v_{i,k}^L=v_{j,k}^R=v^1$ otherwise. The choice for $v^0$ and $v^1$
will be given below.
For $\gamma$, we fix two distributions $\rho_\gamma^0$ and $\rho_\gamma^1$
and fix $\rho^{\gamma_k}=
\rho_\gamma^0$ for $k\leq r$ and $\rho^{\gamma_k}=\rho_\gamma^1$ otherwise.
Then:
\begin{align*}
\inf_{\rho \in \mathcal{F}(M)} \mathcal{R}(\rho,\lambda) & \leq 
\frac{2}{\pi}((K-r)v^1+rv_0)+2r\ell(M)\sqrt{\frac{2v^0}{\pi}}+\frac{1}{\lambda}
\left(r\mathcal{K}(\rho_\gamma^0,\pi^\gamma) + 
(K-r)\mathcal{K}(\rho_\gamma^1,\pi^\gamma) \right)
\\
& \quad +\frac{m_1+m_2}{2\lambda}\Biggl\lbrace  r \underbrace{\left[
\mathbb{E}_{\rho^0_\gamma}\left[\frac{1}{\gamma_k}\right]
(v^0 + \ell^2(M))+\mathbb{E}_{\rho^0_\gamma}
\log \gamma_k-\log v^0-1\right]}_{A_1}
\\
& + (K-r) 
\underbrace{\left[\mathbb{E}_{\rho^1_\gamma}\left[\frac{1}{\gamma_k}\right] v_1+
\mathbb{E}_{\rho^0_\gamma}
\log \gamma_k-\log v^1-1\right]}_{A_2}  
\Biggr\rbrace.
\end{align*}
By actually choosing $\rho_\gamma^0=\pi^\gamma|_{[1,1+\delta]}$ for some
$\delta>0$ and $\rho_\gamma^1=\pi^\gamma|_{[v^1,v_1+\delta]}$, we obtain
\begin{align*}
A_2&\leq 1-1+\log \frac{v^1+\delta}{v^1} 
\leq \frac{\delta}{v^1}; \\
A_1 & \leq v^0+l^2+\delta-\log v^0.
\end{align*}
At this stage, we can set the free parameters $v_0$, $v_1$ and $\delta$
in order to reach the desired
rate. The choices are: $v_1=\frac{1}{n}, v_0 = \frac{1}{n^2}, \delta = 
\frac{r}{Kn}$. We finally have to upper bound $r\mathcal{K}(\rho_\gamma^0,\pi^\gamma) + 
(K-r)\mathcal{K}(\rho_\gamma^1,\pi^\gamma)$. The upper bound actually depends on
the choice for $\pi^\gamma$. We consider three cases: the Gamma prior
with $\alpha \geq 1$, with $\alpha<1$ and then the inverse-Gamma prior.

Let us deal with the $\Gamma^{-1}(\alpha,\beta)$ prior first:
\begin{align*}
&r\mathcal{K}(\rho_\gamma^0,\pi^\gamma) + 
(K-r)\mathcal{K}(\rho_\gamma^1,\pi^\gamma) -K\log \frac{1}{\delta} - K 
\log\frac{\Gamma(\alpha)}{\beta^\alpha}\\
&\leq r[(\alpha+1)\log(1+\delta)+\beta] + (K-r)[(\alpha+1)\log(v_1+\delta) + 
\frac{\beta}{v_1}] \\
&\leq r[(\alpha+1)\delta+\beta] + (K-r)\left[(\alpha+1)(\log v_1 
+\frac{\delta}{v_1}) + \frac{\beta}{v_1}\right]  \\
& \leq r[(\alpha+1)\frac{r}{Kn}+\beta] + K\left[(\alpha+1)(-\log n +1) + n\beta 
\right] 
\end{align*}

Let's turn to the $\Gamma(\alpha,\beta)$ distribution with $\alpha \geq 1$:
\begin{align*}
&r\mathcal{K}(\rho_\gamma^0,\pi^\gamma) + 
(K-r)\mathcal{K}(\rho_\gamma^1,\pi^\gamma) -K\log \frac{1}{\delta} - K 
\log\frac{\Gamma(\alpha)}{\beta^\alpha}\\
& \leq r[\beta(1+\delta)] + (K-r)[-(\alpha-1)\log v_1 +\beta(v_1 + \delta)]\\
& \leq r[\beta(1+\delta)] + (K-r)[(\alpha-1)\log \frac{1}{v_1} +\beta(v_1 + 
\delta)] \\
&\leq 2r\beta+ K[(\alpha-1)\log n +\frac{2\beta}{n}] 
\end{align*}

The last case is the $\Gamma(\alpha,\beta)$ distribution with $0<\alpha < 1$:
\begin{align*}
&r\mathcal{K}(\rho_\gamma^0,\pi^\gamma) + 
(K-r)\mathcal{K}(\rho_\gamma^1,\pi^\gamma) -K\log \frac{1}{\delta} - K 
\log\frac{\Gamma(\alpha)}{\beta^\alpha}\\
&\leq  r[-(\alpha-1)\log(1+\delta) + \beta(1+\delta)] + (K-r)[ 
-(\alpha-1)\log(v_1+\delta) +\beta(v_1 + \delta)]\\
&\leq r[(1-\alpha)\delta + \beta(1+\delta)] + (K-r)[ (1-\alpha)(\log v_1 + 
\frac{\delta}{v_1}) + \beta(v_1 + \delta)] \\ 
& \leq 2r\beta + r(1-\alpha)\frac{r}{Kn} + K[ (1-\alpha)(-\log n + 1) + 
\frac{2\beta}{n}]
\end{align*}

In any case, as $\lambda \leq n$,
when $\alpha$ and $\beta$ are constant, the leading term
is in $\frac{r(m_1+m_2)(\ell^2(M)+ \log n)}{\lambda}$ so we can upper
bound the whole by $\mathcal{C}_{\pi^\gamma}\frac{r(m_1+m_2)(\ell^2(M)+
\log n)}{\lambda}$ where $\mathcal{C}_{\pi^\gamma}$ depends on
$\alpha$ and $\beta$ (and takes a different form depending on the
case: Gamma or inverse-Gamma).
\end{proof}

Note actually that from the previous proof we can provide more explicit
forms for the bound in the three cases. We did not include this in the core
of the paper, but we prove the following lemmas for the sake of completeness.

\begin{lemma}
When $\pi^\gamma = \Gamma(\alpha,\beta)$,
\begin{align*}
\inf_{\rho \in \mathcal{F}(M)} \mathcal{R}(\rho,\lambda) & \leq 
\frac{1}{n}\left[ 
\frac{4}{\pi}K+\sqrt{\frac{8}{\pi}}r\ell(M) \right]+\frac{r(m_1+m_2)}{2\lambda}  
\left[3+\ell^2(M) + 2\log n \right] \\
& \quad +\frac{K}{\lambda}\left[ \log \frac{Kn}{r} +  
\log\frac{\Gamma(\alpha)}{\beta^\alpha} + 
\frac{r}{K}\left[(\alpha+1)\frac{r}{Kn}+\beta\right] + \left[(\alpha+1)(-\log n 
+1) + n\beta \right]  \right].
\end{align*}
\end{lemma}

\begin{lemma}
When $\pi^\gamma = \Gamma^{-1}(\alpha,\beta)$ with $\alpha\geq 1$,
\begin{align*}
\inf_{\rho \in \mathcal{F}(M)} \mathcal{R}(\rho,\lambda) & \leq 
\frac{1}{n}\left[ 
\frac{4}{\pi}K+\sqrt{\frac{8}{\pi}}r\ell(M) \right]+\frac{r(m_1+m_2)}{2\lambda}  
\left[3+\ell^2(M) + 2\log n \right] \\
& \quad +\frac{K}{\lambda}\left[\log \frac{Kn}{r} + 
\log\frac{\Gamma(\alpha)}{\beta^\alpha} + \frac{2r\beta}{K}+ 
\left[(\alpha-1)\log n +\frac{2\beta}{n}\right] \right].
\end{align*}
\end{lemma}

\begin{lemma}
When $\pi^\gamma = \Gamma^{-1}(\alpha,\beta)$ with $0<\alpha< 1$,
\begin{align*}
\inf_{\rho \in \mathcal{F}(M)} \mathcal{R}(\rho,\lambda) & \leq 
\frac{1}{n}\left[ 
\frac{4}{\pi}K+\sqrt{\frac{8}{\pi}}r\ell(M) \right]+\frac{r(m_1+m_2)}{2\lambda}  
\left[3+\ell^2(M) + 2\log n \right] \\
& \quad +\frac{K}{\lambda}\left[\log \frac{Kn}{r} + 
\log\frac{\Gamma(\alpha)}{\beta^\alpha} + \frac{2r\beta}{K} + 
(1-\alpha)\frac{r^2}{K^2n} + \left[(1-\alpha)(-\log n + 1) + 
\frac{2\beta}{n}\right] \right].
\end{align*}
\end{lemma}

\subsection{Proofs of the results in Subsection~\ref{sec:theoretical_bound}}

We first start with preliminary lemmas.

\begin{lemma}
\label{le:theo1}
For $\epsilon >0$, with probability at least $1-\epsilon$ and for every $s\in 
(0,1)$,
\begin{align*}
\overline{r_n}\leq (1+s)\overline{R}+\frac{1}{ns}\log \frac{1}{\epsilon}
\end{align*}
\end{lemma}

\begin{proof}
Let $s\in (0,1)$, then
\begin{align*}
\mathbb{E}\left(\exp [sn \overline{r_n}]\right) &= \prod_{\ell=1}^n 
\mathbb{E}\left(\exp \left[ s \mathbbm{1}(Y_{\ell}M^B_{X_\ell}<0) \right] \right) 
\\
& = \prod_{\ell=1}^n \mathbb{E}\left(\exp \left[ s 
\mathbbm{1}(Y_{\ell}M^B_{X_\ell}<0)+ 0 \left( 1- 
\mathbbm{1}(Y_{\ell}M^B_{X_\ell}<0) 
\right) \right] \right) \\
& \leq \prod_{\ell=1}^n  \left(\left(1- \mathbb{E}\left[ 
\mathbbm{1}(Y_{\ell}M^B_{X_\ell}<0)\right] \right) + e^s \mathbb{E}\left[ 
\mathbbm{1}(Y_{\ell}M^B_{X_\ell}<0)\right] \right) \\
& \leq \prod_{\ell=1}^n  \left((1-\overline{R}) + e^s \overline{R} 
\right) 
 = \prod_{\ell=1}^n  \left(1+\overline{R} (e^s -1) \right)  \\
& \leq \prod_{\ell=1}^n  \exp \left( \overline{R}(e^s-1) \right) = \exp 
\left( n \overline{R}(e^s-1) \right).
\end{align*}
Therefore, for $\epsilon \in (0,1)$:
\begin{align*}
\mathbb{E}\left[ \exp \left( sn\overline{r_n} - n \overline{R}(e^s-1) -\log 
\frac{1}{\epsilon} \right)\right] \leq \epsilon.
\end{align*}
We use the fact that $\mathbb{E}[\exp U] \geq \mathbb{P}(U>0)$
(Markov's inequality) for any $U$ 
so, with probability at least $1- \epsilon$:
\begin{align*}
\overline{r_n} \leq \frac{e^s-1}{s}\overline{R} + \frac{\log 
\frac{1}{\epsilon}}{sn} 
\end{align*}
On $[0,1]$, $e^x\leq 1+x+x^2$ so it's done.
\end{proof}

\begin{lemma}
\label{le:theo2}
Assume that Mammen and Tsybakov's assumption is satisfied for a certain 
constant 
$C$.
Assume that $\lambda<2n/C$.
Then, for $\epsilon >0$, with probability at least $1-\epsilon$:
\begin{align}
\int R d\widetilde{\rho}_\lambda \leq \overline{R} + \frac{1}{1-C\lambda/(2n)} 
\left\lbrace \inf_{\rho \in \mathcal{F}} \left[ r_n^h\left( L^0 
R^{0\top}\right) 
+ \mathcal{R}(\rho,\lambda)\right] - \overline{r}_n + \frac{1}{\lambda} 
\log\left(\frac{1}{\epsilon}\right) \right\rbrace
\end{align}
\end{lemma}

\begin{proof}
Assume that the Mammen and Tsybakov's assumption is satisfied for a certain 
constant C. The $0-1$ loss is bounded then, from
Bernstein's inequality (Theorem 10 page 37 in~\cite{boucheron2013concentration})
we get:
\begin{align*}
\int \mathbb{E}\exp\{\lambda[R(LR^\top)-\overline{R}]-\lambda[
r_n(LR^\top)-\overline{r}_n]\} {\rm d}\pi(L,R,\gamma) \leq 
\int \exp[C\lambda^2/(2n)[R(LR^\top)-\overline{R}]] {\rm d}\pi(L,R,\gamma).
\end{align*}
Apply Fubini's theorem to the inequality:
\begin{align*}
\mathbb{E} \int \exp \{ (\lambda-C\lambda^2/(2n))[R(LR^\top)-\overline{R}] 
-\lambda[r_n(LR^\top)-\overline{r}_n]\}) \pi(d\theta) \leq 1
\end{align*}
(we remind that $\theta=(L,R,\gamma)$ for short).
\begin{align*}
\mathbb{E} \exp \left\lbrace \sup_{\rho} \int [\lambda 
[R(LR^\top)-\overline{R}]-\lambda [r_n(LR^\top)-\overline{r}_n] - 
C\lambda^2/(2n) [R(LR^\top)-\overline{R}] ]\rho (d\theta) - 
\mathcal{K}(\rho,\pi) \right\rbrace \leq 1.
\end{align*}
Using Markov's inequality,
\begin{align*}
\mathbb{P} \left( \sup_{\rho} \int [(\lambda-C\lambda^2/(2n)) 
[R(LR^\top)-\overline{R}]-\lambda [r_n(LR^\top)-\overline{r}_n]  ]\rho 
(d\theta) 
- \mathcal{K}(\rho,\pi) +\log \epsilon>0 \right)  \leq \epsilon.
\end{align*}
Then take the complementary of this event and we get that with probability at 
least $1-\epsilon$:
\begin{align*}
\forall \rho, \qquad (\lambda-C\lambda^2/(2n)) \int  
[R(LR^\top)-\overline{R}]\rho (d\theta) \leq \lambda \int 
[r_n(LR^\top)-\overline{r}_n]  \rho (d\theta) + \mathcal{K}(\rho,\pi) + \log 
\frac{1}{\epsilon}
\end{align*}
Now, note that
\begin{align*}
(\lambda-C\lambda^2/(2n)) \left[ \int R d\rho - \overline{R}\right] &\leq  
\lambda \left[ \int  r_n d\rho - \overline{r}_n\right] + 
\mathcal{K}(\rho,\pi)+\log\left(\frac{1}{\epsilon}\right) \\
\Rightarrow(\lambda-C\lambda^2/(2n)) \left[ \int R d\rho - \overline{R}\right] 
&\leq  \lambda \left[ \int r_n^h d\rho + \frac{1}{\lambda} 
\mathcal{K}(\rho,\pi) 
\right] - \lambda \overline{r}_n+\log\left(\frac{1}{\epsilon}\right) \\
\Rightarrow (\lambda-C\lambda^2/(2n)) \left[ \int R d\rho - \overline{R}\right] 
&\leq  \lambda \left[  r_n^h\left( L^0 R^{0\top}\right) + 
\mathcal{R}(\rho,\lambda) - \overline{r}_n+ \frac{1}{\lambda 
}\log\left(\frac{1}{\epsilon}\right) \right] \\
\end{align*}
As it stands for all $\rho$ then the right hand side can be minimized and the 
minimizer over $\mathcal{F}$ is $\widetilde{\rho}_\lambda$.
Thus we get, when $\lambda<2n/C$,
\begin{align*}
\int R d\widetilde{\rho}_\lambda \leq \overline{R} + \frac{1}{1-C\lambda/(2n)} 
\left\lbrace \inf_{\rho \in \mathcal{F}} \left[ r_n^h\left( L^0 
R^{0\top}\right) 
+ \mathcal{R}(\rho,\lambda)\right] - \overline{r}_n + \frac{1}{\lambda} 
\log\left(\frac{1}{\epsilon}\right) \right\rbrace
\end{align*}
\end{proof}

We are now ready for the proofs.

\begin{proof}[Proof of Theorem~\ref{th:theo}]
We apply Lemma~\ref{le:theo2} and, as we have $\mathcal{F}(M^B)
\subset\mathcal{F}$,
\begin{align}
\int R d\widetilde{\rho}_\lambda \leq \overline{R} + \frac{1}{1-C\lambda/(2n)} 
\left\lbrace \inf_{\rho \in \mathcal{F}(M)} \left[ r_n^h\left( L^0 
R^{0\top}\right) 
+ \mathcal{R}(\rho,\lambda)\right] - \overline{r}_n + \frac{1}{\lambda} 
\log\left(\frac{1}{\epsilon}\right) \right\rbrace
\end{align}
As by definition, all the entries of
$M^B$ are in $\{-1,1\}$, $r_n^h(M^B)=2\overline{r_n}$ and then, by 
Lemma~\ref{lemma:rank}:
\begin{align*}
\int R d\widetilde{\rho}_\lambda \leq \overline{R} + \frac{1}{1-C\lambda/(2n)} 
\left\lbrace  2 \overline{r_n} + 
\mathcal{C}_{\pi^\gamma}\frac{{\rm rank}(M^B)(m_1+m_2)(\ell^2(M)+\log 
n)}{\lambda} + \frac{1}{\lambda} \log\left(\frac{1}{\epsilon}\right) 
\right\rbrace
\end{align*}
Then, we use Lemma~\ref{le:theo2} to get, with probability at least
$1-2\varepsilon$,
\begin{align*}
\int R d\widetilde{\rho}_\lambda  & \leq \overline{R}
+ \frac{1}{1-C\lambda/(2n)} 
\Biggl\lbrace 2(1+s)\overline{R}+\frac{1}{ns}\log \frac{1}{\epsilon} 
\\
& + 
\mathcal{C}_{\pi^\gamma}\frac{{\rm rank}(M^B)(m_1+m_2)(\ell^2(M)+\log 
n)}{\lambda} + \frac{1}{\lambda} \log\left(\frac{1}{\epsilon}\right) 
\Biggr\rbrace
\end{align*}
To end up the proof, we have to take $\lambda=\frac{2cn}{C}$
with $c \in (0,1/2)$. We thus have:
\begin{align*}
\frac{1}{1-C\lambda/(2n)} = \frac{1}{1-c} \leq 1+2c,
\end{align*}
this ends the proof by taking $c=s/2$.
\end{proof}

\begin{proof}[Proof of Corollary~\ref{co:theo}]
 As we are in the noiseless case, the margin assumption is satisfied
 with $C=1$, and $\overline{R}=0$.
\end{proof}

\subsection{Detailed calculations for Subsection~\ref{sec:logistic}}
\label{sec:proof_logistic}

\begin{proof}[Proof of Proposition~\ref{prop:VariationalLaplace}]
From \cite{jaakkola2000}, we have the following lower bound: 
\begin{align*}
&\forall (x, \xi) \in \mathbb{R}^2 \\
&\log \sigma(x) \geq  \log \sigma(\xi) + \frac{x-\xi}{2}-\tau(\xi)(x^2-\xi^2 
) \quad \textrm{where} \quad \tau(\xi)=\frac{1}{2\xi}\left( 
\sigma(\xi)-\frac{1}{2} \right)
\end{align*}

The likelihood of one observation $y \in \lbrace -1,1\rbrace$ at point $x$ is then lower bounded:
\begin{align*}
\forall \xi \in \mathbb{R}, \sigma(yx) &\geq \sigma(\xi)\exp \left\lbrace \frac{yx-\xi}{2}-\tau(\xi)(x^2-\xi^2)\right\rbrace := h(yx,\xi).
\end{align*}

Therefore, the likelihood of the model is lower bounded:
\begin{align*}
\forall \xi \in \mathbb{R}^n, \quad \Lambda(L,R)&=\prod_{l=1}^n \sigma(Y_l (LR^\top)_{X_l}) \\
 &\geq \prod_{l=1}^n \sigma(\xi_l)\exp \left\lbrace \frac{(LR^\top)_{X_l}-\xi_l}{2}-\tau(\xi_l)\left[(LR^\top)_{X_l}^2-\xi_l^2\right] \right\rbrace:=H(\theta,\xi).
\end{align*}
\end{proof}

It is now easy to optimize $\mathcal{L}(\rho,\xi)$ with respect to $\xi$ elementwise, which is the same as maximizing $H(\theta,\xi)$ elementwise and then each part $h(Y_l(LR^\top)_{X_l},\xi_l)$: 
\begin{align*}
\widehat{\xi_l} &= \textrm{arg}\max_{\xi_l} \int \log H(\theta,\xi)\rho(d\theta) = \textrm{arg}\max_{\xi_l} \mathbb{E}_\rho \left[ \log h(Y_l(LR^\top)_{X_l},\xi_l) \right] \\ &=\textrm{arg}\max_{\xi_l} \left\lbrace \log \sigma(\xi_l)- \frac{\xi_l}{2} -\tau(\xi_l) \left( \mathbb{E}_\rho \left[ (LR^\top)_{X_l}^2 \right]-\xi_{ij}^2 \right) \right\rbrace 
\end{align*}

The maximum is reached at the zero of the derivative and we can conclude that:
\begin{align*}
\widehat{\xi_l}^2 &= \mathbb{E}_\rho \left[ (LR^\top)_{X_l}^2 \right]
\end{align*}

\bibliographystyle{plain}

\begin{thebibliography}{10}

\bibitem{AlquierRidgway2015}
P.~{Alquier}, J.~{Ridgway}, and N.~{Chopin}.
\newblock {On the properties of variational approximations of Gibbs
  posteriors}.
\newblock {\em ArXiv e-prints}, June 2015.

\bibitem{Bishop6}
C.~M. Bishop.
\newblock {\em Pattern Recognition and Machine Learning (Information Science
  and Statistics)}.
\newblock Springer-Verlag New York, Inc., Secaucus, NJ, USA, 2006.

\bibitem{boucheron2013concentration}
S.~Boucheron, G.~Lugosi, and P.~Massart.
\newblock {\em Concentration inequalities: A nonasymptotic theory of
  independence}.
\newblock OUP Oxford, 2013.

\bibitem{boyd2011distributed}
S.~Boyd, N.~Parikh, E.~Chu, B.~Peleato, and J.~Eckstein.
\newblock Distributed optimization and statistical learning via the alternating
  direction method of multipliers.
\newblock {\em Foundations and Trends in Machine Learning}, 3(1):1--122, 2011.

\bibitem{Cai2013}
T.~Cai and W.-X. Zhou.
\newblock A max-norm constrained minimization approach to 1-bit matrix
  completion.
\newblock {\em Journal of Machine Learning Research}, 14:3619--3647, 2013.

\bibitem{CandesP10}
E.~J. Cand{\`e}s and Y.~Plan.
\newblock {Matrix Completion With Noise}.
\newblock {\em Proceedings of the IEEE}, 98(6):925--936, 2010.

\bibitem{CandesRecht}
E.~J. Cand{\`e}s and B.~Recht.
\newblock Exact matrix completion via convex optimization.
\newblock {\em Commun. ACM}, 55(6):111--119, 2012.

\bibitem{CandesT10}
E.~J. Cand{\`e}s and T.~Tao.
\newblock The power of convex relaxation: near-optimal matrix completion.
\newblock {\em IEEE Transactions on Information Theory}, 56(5):2053--2080,
  2010.

\bibitem{catoni2004statistical}
O.~Catoni.
\newblock {\em Statistical Learning Theory and Stochastic Optimization}.
\newblock Saint-Flour Summer School on Probability Theory 2001 (Jean Picard
  ed.), Lecture Notes in Mathematics. Springer, 2004.

\bibitem{Catoni2007}
O.~Catoni.
\newblock {\em {PAC}-{B}ayesian supervised classification: the thermodynamics
  of statistical learning}.
\newblock Institute of Mathematical Statistics Lecture Notes---Monograph
  Series, 56. Institute of Mathematical Statistics, Beachwood, OH, 2007.

\bibitem{chatterjee2015}
S.~Chatterjee.
\newblock Matrix estimation by universal singular value thresholding.
\newblock {\em Ann. Statist.}, 43(1):177--214, 02 2015.

\bibitem{Dalalyan2008}
A.~Dalalyan and A.~B. Tsybakov.
\newblock {Aggregation by exponential weighting, sharp {PAC}-{B}ayesian bounds
  and sparsity}.
\newblock {\em Machine Learning}, 72(1):39--61, 2008.

\bibitem{Davenport14}
M.~A. {Davenport}, Y.~{Plan}, E.~{van den Berg}, and M.~{Wootters}.
\newblock {1-Bit Matrix Completion}.
\newblock {\em to appear in Information and Inference}, 2014.

\bibitem{erbrich:graepel:2002}
R.~Herbrich and T.~Graepel.
\newblock A {PAC}-{B}ayesian margin bound for linear classifiers.
\newblock {\em IEEE Transactions on Information Theory}, 48(12):3140--3150,
  2002.

\bibitem{jaakkola2000}
T.~S. Jaakkola and M.~I. Jordan.
\newblock Bayesian parameter estimation via variational methods.
\newblock {\em Statistics and Computing}, 10(1):25--37, 2000.

\bibitem{Klopp2015}
O.~{Klopp}, J.~{Lafond}, E.~{Moulines}, and J.~{Salmon}.
\newblock {Adaptive Multinomial Matrix Completion}.
\newblock {\em ArXiv e-prints}, August 2014.

\bibitem{BayesGroupLasso}
M.~Kyung, J.~Gill, M.~Ghosh, and G.~Casella.
\newblock {Penalized regression, standard errors, and Bayesian lassos}.
\newblock {\em Bayesian Analysis}, 5(2):369--412, 2010.

\bibitem{latouche2015}
Pierre Latouche, St{\'e}phane Robin, and Sarah Ouadah.
\newblock Goodness of fit of logistic models for random graphs.
\newblock {\em arXiv preprint arXiv:1508.00286}, 2015.

\bibitem{LimTeh2007}
Y.~J. Lim and Y.~W. Teh.
\newblock Variational {B}ayesian approach to movie rating prediction.
\newblock In {\em Proceedings of KDD Cup and Workshop}, 2007.

\bibitem{mai2015}
T.~T. Mai and P.~Alquier.
\newblock A bayesian approach for noisy matrix completion: Optimal rate under
  general sampling distribution.
\newblock {\em Electron. J. Statist.}, 9(1):823--841, 2015.

\bibitem{Mammen1999}
E.~Mammen and A.~Tsybakov.
\newblock Smooth discrimination analysis.
\newblock {\em The Annals of Statistics}, 27(6):1808--1829, 1999.

\bibitem{McA}
D.A. McAllester.
\newblock Some {PAC}-{B}ayesian theorems.
\newblock In {\em Proceedings of the Eleventh Annual Conference on
  Computational Learning Theory}, pages 230--234, New York, 1998. ACM.

\bibitem{BayesLasso}
T.~Park and G.~Casella.
\newblock {The Bayesian Lasso}.
\newblock {\em Journal of the American Statistical Association},
  103(482):681--686, 2008.

\bibitem{Recht2013}
B.~Recht and C.~Ré.
\newblock Parallel stochastic gradient algorithms for large-scale matrix
  completion.
\newblock {\em Mathematical Programming Computation}, 5(2):201--226, 2013.

\bibitem{SalMnih2008}
R.~Salakhutdinov and A.~Mnih.
\newblock Bayesian probabilistic matrix factorization using {M}arkov chain
  {M}onte {C}arlo.
\newblock In {\em Proceedings of the 25th International Conference on Machine
  Learning}, pages 880--887, 2008.

\bibitem{seldin2012pac}
Y.~Seldin, F.~Laviolette, N.~Cesa-Bianchi, J.~Shawe-Taylor, and P.~Auer.
\newblock {PAC}-{B}ayesian inequalities for martingales.
\newblock {\em IEEE Transactions on Information Theory}, 58(12):7086--7093,
  2012.

\bibitem{shawe2003pac}
J.~Shawe-Taylor and J.~Langford.
\newblock {PAC}-{B}ayes \& margins.
\newblock {\em Advances in neural information processing systems}, 15:439,
  2003.

\bibitem{srebro2004}
N.~Srebro, J.~Rennie, and T.~S. Jaakkola.
\newblock Maximum-margin matrix factorization.
\newblock In {\em Advances in neural information processing systems}, pages
  1329--1336, 2004.

\bibitem{VC}
V.~Vapnik.
\newblock {\em Statistical Learning Theory}.
\newblock Wiley, 1998.

\bibitem{zhang2004}
T.~Zhang.
\newblock Statistical behavior and consistency of classification methods based
  on convex risk minimization.
\newblock {\em Ann. Statist.}, 32(1):56--85, 02 2004.

\end{thebibliography}

\end{document}